\documentclass[letterpaper]{article}
\usepackage{aaai}
\usepackage{times}
\usepackage{helvet}
\usepackage{courier}
\usepackage[semicolon]{natbib}
\usepackage[ruled,vlined]{algorithm2e}
\usepackage{amsmath}
\usepackage{amsthm}
\usepackage{amsfonts}
\usepackage{amssymb}
\usepackage{mathrsfs}
\usepackage{calligra}
\usepackage{bbm}
\usepackage{tikz}
\usepackage{booktabs}

\usepackage{array}
\usepackage{tabularx}
\usepackage{multirow}
\usepackage{optidef}

\makeatletter
\newcommand{\printfnsymbol}[1]{%
  \textsuperscript{\@fnsymbol{#1}}%
}
\makeatother

\DeclareMathOperator\supp{supp}
\newtheorem{lemma}{Lemma}

\newtheorem{prop}{Proposition}
\newcommand{\norm}[1]{\left\lVert#1\right\rVert}

\begin{document}
%
\title{Discrepancy Minimization in Domain Generalization with Generative Nearest Neighbors}
\author{
Prashant Pandey,\textsuperscript{\rm 1}
Mrigank Raman\thanks{equal contribution},\textsuperscript{\rm 1}
Sumanth Varambally\printfnsymbol{1},\textsuperscript{\rm 1}
Prathosh AP\textsuperscript{\rm 1}\\
\textsuperscript{\rm 1}IIT Delhi\\
\{bsz178495, mt1170736, mt6170855, prathoshap\}@iitd.ac.in
}
\maketitle
\begin{abstract}
\begin{quote}
Domain generalization (DG) deals with the problem of domain shift where a machine learning model trained on multiple-source domains fail to generalize well on a target domain with different statistics. Multiple approaches have been proposed to solve the problem of domain generalization by learning domain invariant representations across the source domains that fail to guarantee generalization on the shifted target domain. We propose a Generative Nearest Neighbor based Discrepancy Minimization (GNNDM) method which provides a theoretical guarantee that is upper bounded by the error in the labeling process of the target. We employ a Domain Discrepancy Minimization Network (DDMN) that learns domain agnostic features to produce a single source domain while preserving the class labels of the data points. Features extracted from this source domain are learned using a generative model whose latent space is used as a sampler to retrieve the nearest neighbors for the target data points. The proposed method does not require access to the domain labels (a more realistic scenario) as opposed to the existing approaches. Empirically, we show the efficacy of our method on two datasets: PACS and VLCS. Through extensive experimentation, we demonstrate the effectiveness of the proposed method that outperforms several state-of-the-art DG methods.
\end{quote}
\end{abstract}

\section{Introduction}
\begin{figure}
\includegraphics[width=.47\textwidth,height=.16\textwidth]{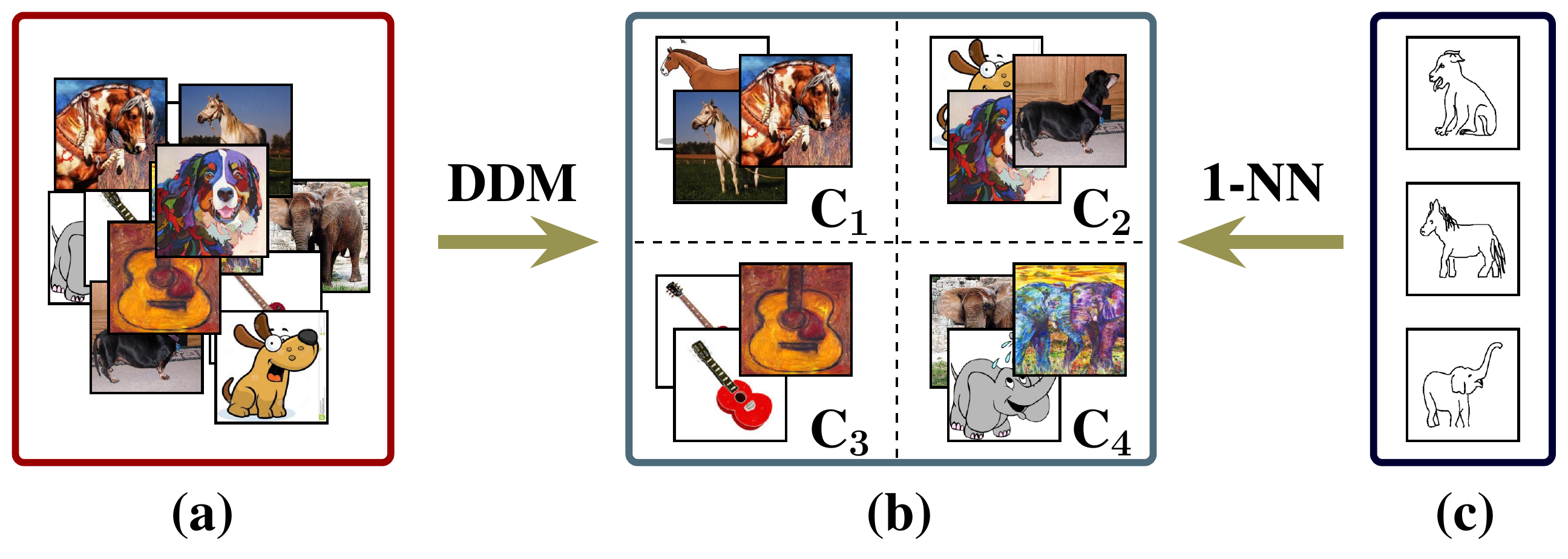}
\caption{(a) is the muti-source data space without domain label information for the data points. We apply Domain Discrepancy Minimization (DDM) to make the divergence zero between the source domains. (b) shows domain agnostic source space with class labels preserved for each data point. (c) is the target data space. For all the target data points, nearest neighbors are sampled from (b) using a generative approach that resembles 1-Nearest Neighbor (1-NN) algorithm whose performance is upper bounded by the labeling error in the target domain. }
\label{fig:dgh}
\end{figure}
Past many years, supervised methods have been employed to solve problems in Machine Learning. Powerful deep neural networks \citep{lee2015deeply} have shown remarkable success in computer vision application where the train and test datasets are assumed to be sampled from the same distribution. In many real world problems, this assumption doesn't hold \citep{torralba2011unbiased} and often the test data (or target domain) is outside the domain of the training dataset (source domain). In an autonomous driving system, a trained model should generalize when the surrounding objects, weather or lighting conditions vary \citep{alcorn2019strike}. Similarly, the medical imaging datasets should generalize across the domains when they are collected and processed at different medical centers under different settings \citep{albadawy2018deep} that include variations in microscope camera, staining process etc. Many solutions exist in literature that facilitate the generalization of target data by reducing the domain shift between source and the target domains. One straight forward strategy is to acquire labels for the target domain to fine-tune the models learned with source data. Since acquisition of labels for every new target domain is costly and time consuming, this approach is practically infeasible. Domain Adaptation (DA) \citep{tzeng2017adversarial, ganin2015unsupervised, hoffman2018cycada, sun2016return, long2016unsupervised, bousmalis2017unsupervised, murez2018image, panareda2017open, pandey2020skin} utilizes labels from the source domain and unlabeled examples from the target dataset to minimize the domain shift. As DA methods require unlabeled target data to retrain the model for every new target domain, its applicability is restricted. 

Domain generalization (DG) \citep{muandet2013domain,li2017deeper,ghifary2015domain,li2018learning,balaji2018metareg,li2019episodic} addresses a harder problem where the model trained on multiple source domains should generalize on completely unseen target domains without the need to re-train the model for every new target domain. Existing methods tackle the DG problem by learning domain invariant feature representations using adversarial methods \citep{li2018domain, li2018deep} or meta-learning approaches where held-out source domains are utilized to simulate the domain shift \citep{li2018learning, balaji2018metareg} or data augmentation techniques \citep{volpi2018generalizing, zhou2020deep} that facilitate procedure for synthesizing data from fictitious target domains and augmenting the source dataset with this synthesized data.

Most of the adversarial methods that produce domain invariant feature representations, lose the class information of the data points during domain alignment. While projecting the features into a domain invariant feature space, the model may not be able to retain the finer details for an image which might include useful class-discriminative information. For example, if the color and texture information is removed from two images containing a sheep and a horse, both the resulting images may look alike although they belong to two different classes. Also, even though the existing meta-learning and data augmentation methods tend to minimize the discrepancy between the source and the target domains empirically, they provide no theoretical guarantee for minimization of the domain shift between the two domains.

In the DG setting, the risk for an arbitrary model has three components: a) source risk on a classifier, b) risk due to divergence across the source domains, c) risk due to divergence between the source and the target domains. We propose a Generative Nearest Neighbor based Discrepancy  Minimization (GNNDM) method that employs Domain Discrepancy  Minimization Network (DDMN) to bring the divergence across the source domains to zero while preserving the class information for the data points. \textit{The idea is to remove the domain bias from each data point by grouping the data points belonging to the same class together irrespective of the domain they belong to}, as shown in Fig. \ref{fig:dgh}. These domain agnostic features are learned to create a generative latent space using a generative model like a Variational Auto-Encoder \citep{kingma2013auto}. For the target data points, we retrieve the nearest neighbors from the generative latent space that improves the performance on the classifier trained with domain agnostic features (from DDMN). We show that by using the nearest neighbor algorithm to retrieve the neighbors for the target examples, \textit{the risk on the target domain is upper bounded by the labeling error in the target domain only.} It should be noted that GNNDM does not assume any specific form for the domain shift between the source and target distributions and it can be arbitrarily large. Also, since it is easier to acquire source data without domain labels \citep{matsuura2020domain}, GNNDM doesn't require them for the task of DG. With these theoretical guarantees, our proposed method is shown to outperform the state-of-the-art DG methods on standard datasets.


\section{Related Work}
\textbf{Meta-learning} : Meta-learning methods aim to improve model robustness against unseen domains by simulating domain shift during training. This is done by splitting the training set into a meta-train and meta-test set. \citep{li2018learning} provide a general framework for meta-learning based domain generalization methods, where model parameters are updated to minimize loss over the meta-train and meta-test domains in a coordinated manner. \citep{balaji2018metareg} propose using a learnt regularizer network which is eventually used to regularize the learning objective of a domain-independent task network. \citep{li2019episodic} train separate feature extractors and classifiers on each of the source domains and minimize the loss on mismatched pairs of feature extractors and classifiers to improve model robustness. However, it is unclear how meta-learning methods generalize to the unseen target domain since they are only trained on the simulated domain shift from the source domains.
\newline \newline
\textbf{Data augmentation}: Data augmentation is known to be an effective method for regularization and improving generalization \citep{hernandez2018further}. Commonly used data augmentation techniques for images include rotation, flipping, random cropping, random colour distortions, amongst others. While these geometric transformations improve generalization in the traditional image classification setting, they cannot account for distributional shifts.
\citep{shankar2018generalizing} use gradients from a domain classifier to perturb images. However, these perturbations are subtle and might not be reflective of practically observed domain shift. \citep{zhou2020deep} aim to address this using an adversarial procedure to train a transformation network, which produces an image translation that aims to generate novel domains while retaining class information. While these domains are quite unlike the source domains, it is not apparent how indicative these generated domain shifts are of practically observed domain differences.\newline \newline
\textbf{Domain-invariant representations}: Another common theme pervasive in domain generalization literature is to transform data into a lower-dimensional domain-invariant representation which retains discriminative class information. \citep{ghifary2015domain} learn an autoencoder to extract domain invariant features by reconstructing inter and cross domain images. \citep{li2018domain} use an adversarial autoencoder, with the Maximum Mean Discrepancy (MMD) measure used to align the representations from the source domains. They then match the latent space to a Laplacian prior using adversarial learning. \citep{dou2019domain} employ episodic training to simulate domain shift, while minimizing a global class alignment loss and local sample clustering objective to cluster points of the same points together. \newline \newline
A noteworthy observation here is that all of the above methods require domain labels, which might not be a practical assumption. \citep{carlucci2019domain} aimed to solve the problem of domain generalization without domain labels by learning the auxiliary task of solving jigsaw puzzles. \citep{matsuura2020domain} explicitly address this issue by using pseudo-labels inferred by clustering the domain discriminative features. They train a domain classifier against these pseudo-labels, which is further used to adversarially train a domain-invariant feature extractor. \citep{motiian2017unified} use a semantic alignment loss that is similar to the cosine similarity loss that is optimized by the DDMN in our work. Unlike their work which uses it as a regularizer loss while training a classifer, we use the DDMN to learn a similarity metric that we further use to obtain the nearest neighbor in the embedding space using an iterative search process during inference. We highlight the advantages of our method through both theoretical results and extensive experimentation and obtain state-of-the-art performance on standard domain generalization benchmarks.
\section{Background and Theory}
\subsection{Preliminaries}
Let $\mathcal{X}$ denote the space from which data originates and let $\mathcal{Y}$ be the set of all possible labels. We denote with $\mathcal{H}$ the space of hypotheses where each hypothesis $h$ in $\mathcal{H}$ maps a point from $\mathcal{X}$ to a label from $\mathcal{Y}$. We will use the following definitions for our analysis:
\begin{itemize}
\item Loss Function: Denoted by $\mathcal{L}$, it quantifies how different $h(x)$ is from its original label $y\in \mathcal{Y}$ for a given data point $x \in \mathcal{X}$.
\newline
\item Domain: A domain is defined by the tuple $(\mathcal{D}, g_\mathcal{D})$ where $\mathcal{D}$ is a probability distribution over $\mathcal{X}$ and $g_\mathcal{D} :\mathcal{X} \to \mathcal{Y}$ is a function that assigns labels. For brevity, throughout this paper we denote a domain $(\mathcal{D}, g_\mathcal{D})$ with $\mathcal{D}$.
\newline
\item Risk: Given a hypothesis $h \in \mathcal{H}$ and a domain $\mathcal{D}$ we define the risk of the hypothesis $h$ on the domain $\mathcal{D}$ as :
 \begin{equation}
 \mathcal{R}_{\mathcal{D}}[h] = \mathbb{E}_{x \sim \mathcal{D}} [\mathcal{L}(x, g(x))]
\end{equation}
\item $\mathcal{H}$ divergence: \citep{kifer2004detecting} introduced the $\mathcal{H}$ divergence for quantifying the domain shift between two domains. Given a hypothesis space $\mathcal{H}$ and two domains $\mathcal{D}$ and $\mathcal{D'}$ the $\mathcal{H}$ divergence is defined as:
\begin{equation}
\Delta_{\mathcal{H}}(\mathcal{D}, \mathcal{D'}) = 2\sup_{h \in \mathcal{H}}\left|\mathbf{P}_{x \in \mathcal{D}}[h(x) = 1] - \mathbf{P}_{x \in \mathcal{D'}}[h(x) = 1]\right|
\end{equation}
\end{itemize}
With these definitions we will now move on to define the domain generalization task. In this task we have a total of N domains out of which there are $|S|$ source domains and $|T|$ target domains. The source domains are denoted by $\mathcal{D}_{i}^{S}$ where $i \in [|S|]$ and the target domains are denoted by $\mathcal{D}_{j}^{T}$ where $j \in [|T|]$. The objective is to train a classifier on the source domains that can predict the labels of the examples in the target domains.

\subsection{Theoretical Analysis}
As we are working on the problem of domain generalization, we consider the labeling functions of all the domains to be identical denoted by $g: \mathcal{X} \to \mathcal{Y}$. Without loss of generality we consider the problem to be binary classification in our theoretical analysis.
We use existing generalization bounds for risk on the unseen domain $\mathcal{D}^{T}$ from previous work as the starting point. Specifically, we use the following inequality from  \citep{albuquerque2019generalizing}.
\begin{prop}
Define the convex hull $\Omega_S$ as the set of mixture distributions $\Omega_S = \{\overline{\mathcal{D}} : \overline{\mathcal{D}} = \sum_{i=1}^{|S|} \alpha_i \mathcal{D}_{i}^S, 0 \leq \alpha_i \leq 1, \sum_{i=1}^{|S|} \alpha_i = 1 \}$. Further, define $\overline{D}_j^T = \arg\min_{\alpha_1, ..., \alpha_{|S|}} \Delta_\mathcal{H} [{\mathcal{D}_{j}^{T}}, \sum_{i=1}^{|S|}\alpha_{i,j} \mathcal{D}_i^S]$. Then, we have:
\begin{equation}
{\mathcal{R}_{\mathcal{D}_j^T}}[h] \leq \sum_{i=1}^{|S|} \alpha_{i,j} {\mathcal{R}_{\mathcal{D}_i^S}}[h] + \frac{\delta+\lambda}{2} + \mu_{\alpha_j}    
\end{equation}
where $\delta = \sup_{i,k \in [|S|]} \Delta_{\mathcal{H}}[\mathcal{D}_i^S, \mathcal{D}_k^S]$, $\lambda = \Delta_{\mathcal{H}}[\mathcal{D}_j^T, \overline{\mathcal{D}}_j^T]$ and $\mu_{\alpha_j}$ is the minimum sum of risks achieved by some $\eta \in \mathcal{H}$ on $\mathcal{D}_j^T$ and $\overline{\mathcal{D}}_j^T$  
\end{prop}


We examine the two terms $\delta$ and $\lambda$. $\delta$ is a measure of the intra-source domain seperation, while $\lambda$ measures the divergence betweeen the source domains and the target domains. In order to minimize $\delta$, we introduce the DDMN (\textbf{D}omain \textbf{D}iscrepancy \textbf{M}inimization \textbf{N}etwork). The DDMN aims to learn identical representations for images having the same label across domains, while learning dissimilar representations for images having different labels. To characterize the action of the DDMN, we introduce the following proposition.\newline

We define the following quantities 
\begin{equation}
\mathbf{P}(\mathcal{X}_i^{[\ell]}) = \mathop{\mathbb{E}}_{x \sim \mathcal{D}_i^S}[\mathbbm{1}_{g(x)=\ell}]\ \forall\ i \in [|S|] 
\end{equation}

\begin{prop}
Given that the following conditions hold:
\begin{enumerate}
    \item There exists a metric space denoted by $(\mathcal{M},d)$ and a transformation function $f : \mathcal{X} \to \mathcal{M}$ such that $d(f(x), f(y)) = 0  \iff g(x)=g(y)$, i.e. $x$ and $y$ have the same labels (irrespective of domain)
    \item All the different classes are equally likely in all the source domains, i.e.
    \begin{equation}
        \mathbf{P}(\mathcal{X}_i^{[\ell]}) = \mathbf{P}(\mathcal{X}_k^{[\ell]}) \quad \forall i, k \in \{1, ..., N\}, \forall \ell \in \mathcal{Y}.
     \end{equation}
\end{enumerate}
Then, 
\begin{enumerate}
    \item $\delta = \sup_{i,k \in [|S|]} \Delta_{\mathcal{H'}}[f(\mathcal{D}_i^S), f(\mathcal{D}_k^S)] = 0$, where $\mathcal{H'}$ denotes the space of hypotheses $h' : \mathcal{M} \to \mathcal{Y}$, and $f(\mathcal{D})$ denotes the distribution $\mathcal{D}$ under the transformation $f$ 
\end{enumerate}
\end{prop}

\begin{proof}
For any $i,k\in[|S|]$ and any $h \in \mathcal{H'}$,
\begin{align}
    \left| \mathbf{P}_{x \in \mathcal{D}_i^S}[h(f(x)) = 1]- \mathbf{P}_{x \in \mathcal{D}_k^S}[h(f(x)) = 1]\right|
\end{align}
\begin{equation}
     =\left| \mathop{\mathbb{E}}_{x \in \mathcal{D}_i^S} [\mathbbm{1}_{h(f(x)) = 1}] - \mathop{\mathbb{E}}_{x \in \mathcal{D}_k^S} [\mathbbm{1}_{h(f(x)) = 1}] \right|
\end{equation}
\newline
\newline
Now, note that $g(x)=g(y)\implies d(f(x), f(y)) = 0 \implies f(x)=f(y)$
 $\implies h(f(x))=h(f(y))$. 
\newline 
This means that all the examples of a given class are assigned the same label $\ell$ by $h$, regardless of the source domain. \newline \newline Depending on the labels assigned to the different classes by $h$, these three cases would arise:
\newline
\newline
\textbf{Case 1:} $h(f(x)) = 1 \quad \forall x \in \mathcal{X}$
\begin{equation*}
    \implies \mathop{\mathbb{E}}_{x \in \mathcal{D}_i^S} [\mathbbm{1}_{h(f(x)) = 1}] = \int_{\mathcal{X}} \mathcal{D}_i^S(x) dx = 1 
\end{equation*}
\begin{equation}
    = \int_{\mathcal{X}} \mathcal{D}_k^S(x) dx = \mathop{\mathbb{E}}_{x \in \mathcal{D}_k^S} [\mathbbm{1}_{h(f(x)) = 1}]
\end{equation}
\begin{equation*}
    \therefore \left|\mathbf{P}_{x \in \mathcal{D}_i^S}[h(f(x)) = 1]- \mathbf{P}_{x \in \mathcal{D}_k^S}[h(f(x)) = 1]\right| = 0    
\end{equation*}
\newline
\textbf{Case 2:} $h(f(x)) = 0 \quad \forall x \in \mathcal{X}$
\begin{equation}
    \mathop{\mathbb{E}}_{x \in \mathcal{D}_i^S} [\mathbbm{1}_{h(f(x)) = 1}] = \mathop{\mathbb{E}}_{x \in \mathcal{D}_j^S}[\mathbbm{1}_{h(f(x)) = 1}] = 0    
    \end{equation}
\begin{equation*}
    \therefore \left|\mathbf{P}_{x \in \mathcal{D}_i^S}[h(f(x)) = 1]- \mathbf{P}_{x \in \mathcal{D}_k^S}[h(f(x)) = 1]\right| = 0
\end{equation*}
\newline
\textbf{Case 3:} $h(f(x)) = 1 \iff  g(x)=\ell$ for some $\ell \in \mathcal{Y}$
\begin{equation}
    \mathop{\mathbb{E}}_{x \in \mathcal{D}_i^S} [\mathbbm{1}_{h(f(x)) = 1}] = \mathop{\mathbb{E}}_{x \in \mathcal{D}_i^S} [\mathbbm{1}_{g(x)=\ell}]  = \mathbf{P}(\mathcal{X}_i^{[\ell]})    
\end{equation}

Similarly, 
\begin{equation}
\mathop{\mathbb{E}}_{x \in \mathcal{D}_k^S} [\mathbbm{1}_{h(f(x)) = 1}] = \mathbf{P}(\mathcal{X}_k^{[\ell]})
\end{equation}
\begin{multline}
    \therefore\left|\mathbf{P}_{x \in \mathcal{D}_i^S}[h(f(x)) = 1]- \mathbf{P}_{x \in \mathcal{D}_k^S}[h(f(x)) = 1]\right| \\
 = \left| \mathbf{P}(\mathcal{X}_i^{[\ell]}) - \mathbf{P}(\mathcal{X}_k^{[\ell]}) \right| 
 = 0  
\end{multline}
(from condition 2)
\newline
\newline
Since 
\begin{equation}
\left|\mathbf{P}_{x \in \mathcal{D}_i^S}[h(f(x)) = 1]- \mathbf{P}_{x \in \mathcal{D}_k^S}[h(f(x)) = 1]\right|=0 \ \forall h \in \mathcal{H'}
\end{equation}
we have that 
\begin{multline}
\Delta_\mathcal{H'}[f(\mathcal{D}_i^S), f(\mathcal{D}_k^S) ] = \\\sup_{h \in \mathcal{H'}}\left|\mathbf{P}_{x \in \mathcal{D}_i^S}[h(f(x)) = 1]- \mathbf{P}_{x \in \mathcal{D}_k^S}[h(f(x)) = 1]\right| = 0
\end{multline}
\newline
\newline
Since $\Delta_\mathcal{H'}[f(\mathcal{D}_i^S), f(\mathcal{D}_k^S) ] = 0$ for all $i, k \in [|S|]$, it follows that 
\begin{equation}
\delta = \sup_{i, k \in [|S|]}\Delta_\mathcal{H'}[f(\mathcal{D}_i^S), f(\mathcal{D}_k^S) ] = 0
\end{equation}
\end{proof}
Thus, using a transformation $f$ satisfying the conditions in Proposition 2, the multi-source Domain Generalization problem can be reduced to a single source problem. Throughout this paper the transformation of the source domains under $f$ is referred to as the transformed source domain. However, we are still left with the task of generalizing to the unseen domain. This is achieved through the Nearest Neighbor Sampling, which is inspired by the 1-Nearest Neighbor algorithm \citep{Cover1967NearestNP}. 
\newline 
We mathematically analyze the effectiveness of this procedure in the task of domain generalization. To this end, we prove the following upper bound on the risk on the target domain.
\begin{prop}
Given a source domain $f(\mathcal{D^S})$, target domain $f(\mathcal{D^T})$ and given that we use a norm-based loss function, we have :
\begin{equation}
{\mathcal{R}_{f(\mathcal{D^T})}}[h] \leq  {{\mathcal{R}_{f(\mathcal{D^S})}}[h]} + {2B^{*}(1 - B^{*})}
\end{equation}
where h is a given hypothesis and $B^{*}$ denotes the Bayes risk on $f(\mathcal{D^T})$
\end{prop}
\begin{proof}
Let $x_T$ denote that point from the target domain $\mathcal{D^T}$ to be classified and let $h$ denote a hypothesis. Let $x_S$ denote the nearest neighbour of $x_T$ in the source domain $f(\mathcal{D^S})$ w.r.t some metric $d$. \citep{pandey2020target} have shown the existence and convergence of such a nearest neighbour in the source domain.\\ 
First, note that
\begin{multline}
\left| g(f(x_T)) - h(f(x_S)) \right| \leq \left| g(f(x_S)) - h(f(x_S)) \right| + \\ \left| g(f(x_T)) - g(f(x_S)) \right|
\end{multline}
Taking expectation on both sides, it follows from Linearity of Expectation that
\begin{align}
    {\mathcal{R}_{f(\mathcal{D^T})}}[h] &= \mathbb{E}_{\mathcal{D^T}} \left| g(f(x_T)) - h(f(x_S)) \right|
    \\
    &\leq \mathbb{E}_{\mathcal{D^S}} \left| g(f(x_S)) - h(f(x_S)) \right|\\& + \mathbb{E}_{\mathcal{D^T}} \left| g(f(x_T)) - g(f(x_S)) \right|   \\
    &= {\mathcal{R}_{f(\mathcal{D^S})}}[h] + \mathbb{E}_{\mathcal{D^T}} \left| g(f(x_T)) - g(f(x_S)) \right| 
\end{align}
The first term indicates the risk due to misclassification by the source domain classifier $h$, while the second term represents the error due to the latent search procedure.
\citep{Cover1967NearestNP} prove that if $x_S \xrightarrow{} x_T$ with probability 1 then: 
\begin{equation}
\mathbb{E}_{\mathcal{D^T}} \left| g(f(x_T)) - g(f(x_S)) \right| \leq 2B^{*}(1 - B^{*})    
\end{equation}
where $B^{*}$ is the Bayes Risk on $f(\mathcal{D^T})$
\newline
It therefore follows that
\begin{equation}
{\mathcal{R}_{f(\mathcal{D^T})}}[h] \leq  {\mathcal{R}_{f(\mathcal{D^S})}}[h] + 2B^{*}(1 - B^{*})    
\end{equation}
\end{proof}
The Bayes Risk quantifies the inherent uncertainty in the data generation process. To better interpret this bound, we introduce the following lemma.
\begin{lemma}
For a joint distribution $\mathcal{D}_{\mathcal{X},\mathcal{Y}}$, if \newline $\text{Var}(Y|X=x)<\sigma^2 \quad \forall x \in \supp({\mathcal{D}_{\mathcal{X},\mathcal{Y}}})$, then we have 
\begin{equation}
2B^*(1-B^*) < 2\sigma^2
\end{equation}
where $B^*$ represents the Bayes Risk over ${\mathcal{D}_{\mathcal{X},\mathcal{Y}}}$
\end{lemma}

\begin{proof}
Suppose $Y | X=x \sim \text{Bernoulli}(p)$.
\begin{align*}
&\text{Var}(Y|X=x)<\sigma^2 \quad \forall x \in \mathcal{X} \\
&\implies p(1-p) < \sigma^2\\
&\implies p \in \left(0, \frac{1}{2} - \sqrt{\frac{1}{4} - \sigma^2}\right) \cup \left(\frac{1}{2} + \sqrt{\frac{1}{4} - \sigma^2}, 1 \right) \\ 
&\therefore \min(p, 1-p) < \frac{1}{2} - \sqrt{\frac{1}{4} - \sigma^2}
\end{align*}

Now, note that
$B^* = \mathbb{E}[\min(\eta_1(x), 1-\eta_1(x))]$, where $\eta_1(x)=\mathbf{P}(Y=1|X=x) = p$
\begin{align}
&\implies B^* = \mathbb{E}[\min(p, 1-p)] \nonumber\\
&= \int_{x \in \supp({\mathcal{D}_{\mathcal{X},\mathcal{Y}}})}{\min(p, 1-p) \; dP_x} \nonumber\\
&< \int_{x \in \supp({\mathcal{D}_{\mathcal{X},\mathcal{Y}}})}\left(\frac{1}{2} - \sqrt{\frac{1}{4} - \sigma^2}\right)  dP_x \nonumber\\
&= \left(\frac{1}{2} - \sqrt{\frac{1}{4} - \sigma^2}\right)\label{eq:bayes_ineq}
\end{align}
Note that $B^* \leq 0.5$, and that the transformation $\kappa(x) = 2x(1-x)$ is monotonically increasing for $x \in [0, 0.5]$. Applying this transformation on both sides of the inequality \eqref{eq:bayes_ineq} we get
\begin{equation}
    2B^*(1-B^*) < 2\sigma^2
\end{equation}
\end{proof}
Note that using Lemma 1, the inequality from Proposition 3 becomes:
\begin{equation}
\mathcal{R}_{f(\mathcal{D^T})}[h] \leq  \mathcal{R}_{f(\mathcal{D^S})}[h] + 2\sigma^2
\end{equation}
Notice that the upper bound only depends on the source risk which we minimize by training a classifier on the source domain and the variance in the inherent data generating process in the target domain, unlike the bound from Proposition 1 which depends on $\lambda$, the distributional divergence between the source and target domains. Further, we argue that the variance $\sigma^2$ is low in most practical applications, since we assume that the underlying data-generation process is well-behaved. For example, in an image classification task on MNIST digits, the inherent uncertainty in the true digit given the image is quite low.


\begin{figure}
\includegraphics[width=.42\textwidth,height=.34\textwidth]{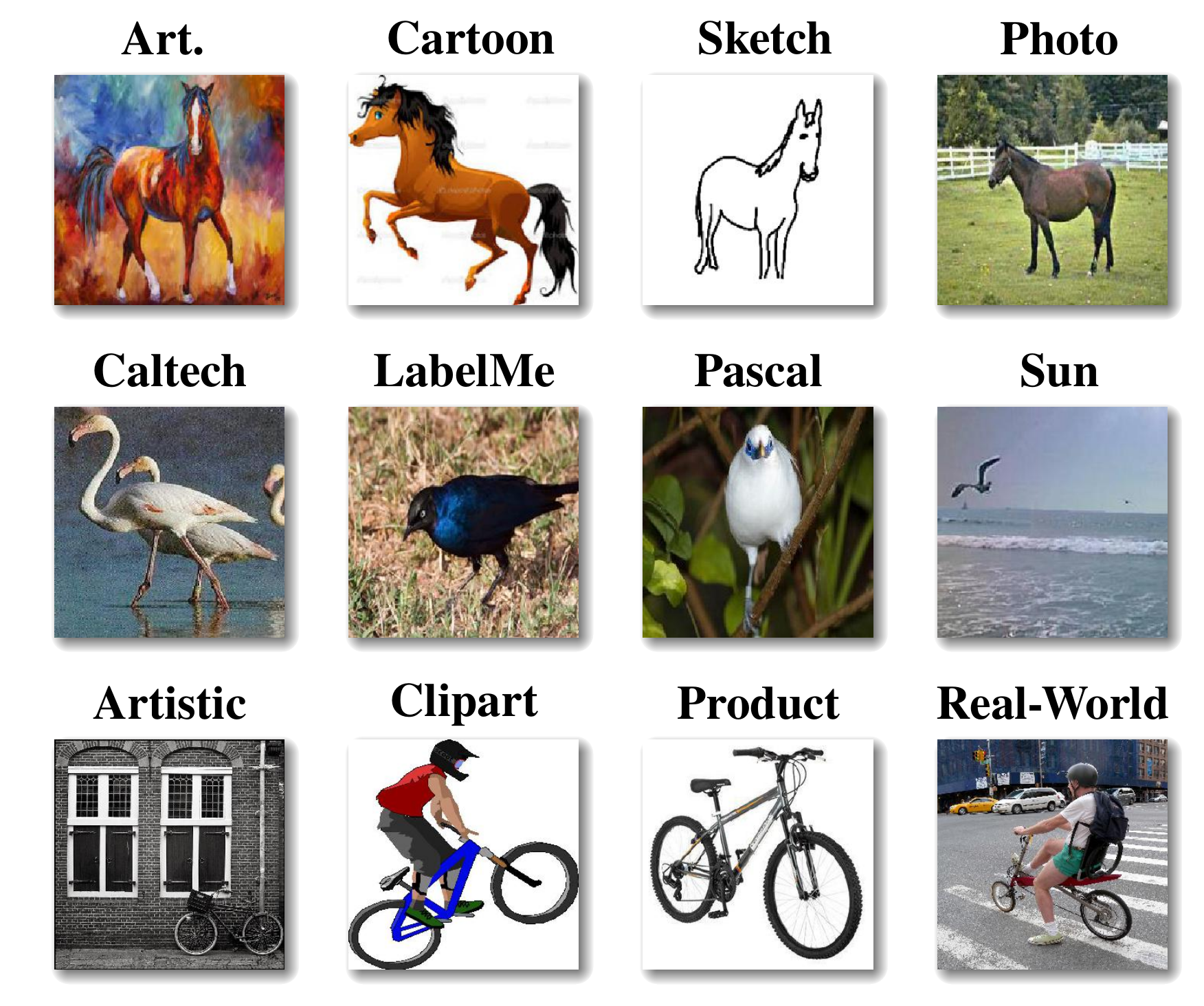}
\caption{Few example images from PACS (1st row), VLCS (2nd row) and Office-Home (3rd row) datasets.}
\label{fig:dg_samples}
\end{figure}

\section{Proposed Method}
In this section, we explain our proposed idea in detail. An overview of our method can be seen in (figure no.). Given a collection of source domains and a collection of target domains, we train a classifier on the source domains and test it on the target domains. We propose to learn a domain agnostic feature space for the source images where the divergence between the different source distributions is minimized. We achieve this using the Domain Discrepancy Minimization Network, which aims to learn such a transformation by parameterizing it as a neural network. Once we have learnt a domain agnostic feature space we further try to minimize the target risk by querying a nearest neighbor from the transformed source domain to a given target feature. To guarantee better convergence, we propose using a generative model (a variational autoencoder in this case) that can sample from the transformed source distribution, thus effectively giving us infinite samples from the transformed source domain. The details of the DDMN and the Nearest Neighbour Sampler (NNS) are provided in the following subsections.  
\begin{algorithm}

\SetAlgoLined

\textbf{Input:} Batch size $N$, structure of similarity network $f$ with parameters $\theta$, constant $\tau$;\\
\KwResult{Trained similarity network $f$ }
 \For{ \text{sampled minibatch} $\{(x_k, y_k)\}^N_{k=1}$ } {
  
\For{all $i \in \{1,...N\}$ and $j \in \{1,...N\}$ } {
$ s_{i,j} = \text{sim}(f(x_i), f(x_j))$
}
\textbf{define} $y_{i,j} =  \begin{cases}
    1, & \text{if } y_i = y_j\\
    0, & \text{otherwise}
\end{cases}$

\textbf{define} $p_{i,j} = \text{sigmoid}(s_{i,j}/\tau)$\\
$\mathcal{L_\theta} = -\frac{1}{N^2} \sum_{i=1}^{N} \sum_{j=1}^{N} (y_{i,j} \log(p_{i,j}) + (1-y_{i,j}) \log(1-p_{i,j}))$\\
\text{Update parameters} $\theta$ \text{to minimize } $\mathcal{L_\theta}$
}
\caption{Domain Similarity Network training}

\end{algorithm}

\begin{algorithm}
\SetAlgoLined
\textbf{Input:} Target image representation $\chi_T$, trained representation VAE decoder $g_\phi$, learning rate $\eta$;\\
\KwResult{``Nearest Neighbour" representation $\hat{\chi}$}

 Sample $\hat{z}$ from $\mathcal{N}(0,1)$;\\
 \Repeat{ \text{convergence of } $\hat{z}$ } {
 $\hat{\chi} = g_\phi(\hat{z})$\\
 $\mathcal{L}_\chi = 1 - \text{sim}(\hat{\chi}, \chi_T)$\\
 $\hat{z} = \hat{z} - \eta \nabla_{\hat{z}} \mathcal{L}_\chi$\\
 
  }
\caption{Nearest Neighbor Sampler}

\end{algorithm}
\subsection{Domain Discrepancy Minimization}
In this subsection we provide an overview of the proposed novel Domain Discrepancy Minimization technique. Our proposed method reduces the inter source domain divergence without using domain labels. We aim to solve the following optimization problem on the combined source domain examples so as to learn a transformation function $f:\mathcal{X} \to \mathcal{M}$ where $\mathcal{M}$ is a domain-agnostic feature space. We assume there are $N$ training examples.
\begin{argmini}
   {f}{\sum_{j=1}^{N}{{\sum_{i=1}^{N}{ (-1)^{\alpha(i,j)}.\norm{f(x_{i}) - f(x_{j})}^{2}}}}}
    {}{}
    \addConstraint{\norm{f(x_{k})}}{=1\ \forall k \in [N]}
\end{argmini}
where \[ \alpha(i,j) = \begin{cases} 
          0 & g(x_{i})=g(x_{j}) \\
          1 & \text{otherwise} 
       \end{cases}
    \]
We note that this problem reduces to minimizing the cosine-similarity between $f(x_i)$ and $f(x_j)$. Thus, we optimize the loss function:
\begin{equation}
    \mathcal{L}_\theta =\frac{1}{N^2} \sum_{i=1}^N \sum_{j=1}^N  (-1)^{\alpha(i,j)}  \frac{f(x_i) \cdot f(x_j)}{\left\lVert f(x_i) \right\rVert\left\lVert f(x_j) \right\rVert} 
\end{equation}
We use a neural network, parameterized by $\theta$ to approximately learn the transformation $f$. We now discuss the architecture of the DDMN.
\newline
We use a deep convolutional network (like ResNet 18 or AlexNet pretrained on ImageNet) as the backbone network and extract features from its last layer. For every pair of feature vectors in a batch, we maximize or minimize the cosine similarity between the representations depending upon whether they have similar or dissimilar labels. Thus we expect images from the same class to have similar feature vectors irrespective of domains. Algorithm 1 highlights the algorithm used to train the DDMN. 
\subsection{Nearest Neighbour Sampler}
While the DDMN effectively minimizes the inter-source domain divergence, there are no guarantees on it's effectiveness on the target domain. To ensure effective generalization on the target domain we use the Nearest Neighbor Sampler (NNS) that can sample from the source domain feature space. \newline
We use a VAE (Variational Autoencoder) that can act as a sampler from the source domain feature space allowing infinite sampling. We use the decoder of the VAE to query the nearest neighbour to a given target feature vector. Algorithm 2 demonstrates our procedure for sampling the nearest neighbour for a target feature vector. 
Once we have trained the DDMN and the NNS, we train a simple 2 hidden-layer feed forward neural network on the source domain feature vectors. During inference on a particular target example, we first extract it's feature vector using the DDMN, then we query it's nearest neighbour using the NNS. We then infer the class of the target example from the sampled nearest neighbour through the trained classifier. 
\section{Experiments and Results}
\begin{table}[hbt!]
\centering
  \scalebox{0.99}{
  \begin{tabular}{lccccc}
    \toprule
        {PACS} & {Art.}
        & {Cartoon}
        & {Sketch} & {Photo} & {Avg.}
        \\
    \midrule
    \multicolumn{6}{c}{AlexNet}\\
    \midrule
    Deep All&65.96&69.50&59.89&89.45&71.20\\
    D-SAM&63.87&70.70&64.66&85.55&71.20\\
     Epi-FCR&64.70&72.30&65.00&86.10&72.02\\
     MetaReg&69.82&70.35&59.26&\textbf{91.07}&72.62\\
    Jigen&67.63&71.71&65.18&89.00&73.38\\

    MMLD&69.27&72.83&66.44&88.98&74.38\\
    
    MASF&70.35&72.46&67.33&90.68&75.21\\
    Ours&\textbf{71.23}&\textbf{74.16}&\textbf{70.81}&90.90&\textbf{76.77}\\
    \midrule
    \multicolumn{6}{c}{ResNet-18}\\
    \midrule
    Deep All&77.65&75.36&69.08&95.12&79.30\\
    Jigen&79.42&75.25&71.35&96.03&80.51\\
    CSD&79.79&75.04&72.46&95.45&80.69\\
    D-SAM&77.33&72.43&77.83&95.30&80.72\\
    MASF&80.29&77.17&71.69&94.99&81.03\\
    Epi-FCR&82.10&77.00&73.00&93.90&81.50\\
    MetaReg&83.70&77.20&70.30&95.50&81.70\\
    MMLD&81.28&77.16&72.29&96.09&81.83\\
    
    DDAIG&84.20&78.10&74.70&95.30&83.10\\
    Ours&\textbf{84.32}&\textbf{79.68}&\textbf{80.21}&\textbf{96.69}&\textbf{85.22}\\
    \bottomrule
  \end{tabular}
  }
  \caption{Comparison of performance between different models using AlexNet and ResNet-18 on PACS.}
  \label{table:pacs}
\end{table}

\begin{table}[hbt!]
\centering
  \scalebox{0.92}{
  \begin{tabular}{lccccc}
    \toprule
        {VLCS} & {Caltech}
        & {LabelMe}
        & {Pascal} & {Sun} & {Avg.}
        \\
    \midrule
    \multicolumn{6}{c}{AlexNet}\\
    \midrule
    Deep All&96.45&60.03&70.41&62.63&72.38\\
    MDA&92.76&62.34&65.25&63.54&70.97\\
    MMD-AAE&94.40&62.60&67.60&64.40&72.30\\
     Epi-FCR&94.10&64.30&67.10&65.90&72.90\\
    Jigen&96.93&60.90&70.62&64.30&73.19\\

    MMLD&96.66&58.77&71.96&68.13&73.88\\
    
    MASF&94.78&64.90&69.14&67.64&74.11\\
    Ours&\textbf{97.09}&\textbf{65.01}&\textbf{73.17}&\textbf{68.22}&\textbf{75.87}\\
    
    
    \bottomrule
  \end{tabular}
  }
  \caption{Comparison of performance between different models using AlexNet backbone on VLCS. Unlike dataset like PACS where domains differ in image styles, VLCS domains are all photos. Our method outperforms SOTA methods even if domains do not vary in image styles and are similar pointing to the fact that the proposed method doesn't assume any form for the domain shift.}
  \label{table:vlcs}
\end{table}
In this section we provide the details on the experiments conducted by us and the underlying results. We use the following standard DG datasets to demonstrate our results: \newline \newline
\textbf{PACS: } Stands for \textbf{P}hoto, \textbf{A}rt painting, \textbf{C}artoon and \textbf{S}ketch. PACS has one of the highest inter-source domain divergences among all the standard DG datasets. Each one of the 9991 images has been assigned 1 out of the 7 possible class labels. We follow the experimental procedure defined in \citep{li2017deeper} wherein we train a classifier on 3 out of 4 domains and test it on the fourth domain. We achieve state-of-the-art results on PACS using both AlexNet and ResNet-18 backbones as is evident in Table 1. We improve upon the current state-of-the-art baselines an each and every domain using a ResNet-18 architecture.
\newline \newline
\textbf{VLCS: } Stands for \textbf{V}OC2007(Pascal), \textbf{L}abelMe, \textbf{C}altech and \textbf{S}un. This dataset is an aggregation of 4 different datasets that are treated as different domains and each image is assigned  1 out of 5 class labels. We follow the same experimental setup as mentioned in \citep{matsuura2020domain} where we train on 3 domains with 70\% data from each and test on the whole fourth domain. We achieve state-of-the-art results on VLCS using the AlexNet backbone as is evident in Table 2. We improve upon the current state-of-the-art baselines on each and every domain which demonstrates the effectiveness of our method. \newline \newline
\section{Conclusion}
We provide a theoretical upper bound for the risk in the problem of domain generalization with multiple source and target domains that is significantly tighter than the existing state-of-the-art methods. We have shown that while preserving the class information, the $\mathcal{H}$-divergence across the source domains can be made zero. Also, the risk on the target domain is shown to be upper bounded by the variance in the labeling procedure of the target domain. We employ a Domain discrepancy minimization network to preserve class information in the source domains while minimizing the divergence across the source domains to zero. Using a generative nearest neighbor sampling algorithm, the neighbors for the target samples are retrieved that helps to reduce the domain shift between source and the target domains and thereby increasing the performance of the classifier learned with domain agnostic features from the Domain discrepancy minimization network. 
\bibliographystyle{aaai}
\bibliography{mybibliography}
\end{document}